\documentclass{article}
\usepackage{amsmath}
\usepackage{amssymb}
\usepackage{amsthm}
\usepackage{xcolor}
\usepackage[authoryear]{natbib}
\usepackage{url}

\newcommand{\DF}[2]{D_F[#1||#2]}

\newcommand{\E}[1]{\mathbb{E}\left[#1\right]}
\newcommand{\Ex}[2]{\mathbb{E}_{#1}\left[#2\right]}

\newcommand{\inner}[2]{\left\langle #1, #2 \right\rangle}

\newtheorem{theorem}{Theorem}[section]
\newtheorem{lemma}[theorem]{Lemma}
\newtheorem{definition}[theorem]{Definition}

\date{\vspace{-5ex}}
\begin{document}

\title{A Generalized Bias-Variance Decomposition for Bregman Divergences}
\author{David Pfau\\\texttt{david.pfau@gmail.com}}
\maketitle

\begin{abstract}
    The bias-variance decomposition is a central result in statistics and machine learning, but is typically presented only for the squared error. We present a generalization of the bias-variance decomposition where the prediction error is a Bregman divergence, which is relevant to maximum likelihood estimation with exponential families. While the result is already known, there was not previously a clear, standalone derivation, so we provide one for pedagogical purposes. A version of this note previously appeared on the author's personal website without context \citep{pfau2013generalized}. Here we provide additional discussion and references to the relevant prior literature.
\end{abstract}

\section{Introduction}
The bias-variance decomposition/dilemma/tradeoff \citep{geman1992neural} states that for any learning algorithm, its generalization error can be broken down into three terms -- a {\em bias} that captures the tendency of a model to favor certain solutions in the absence of enough data, a {\em variance} that captures the tendency of a model to fit fluctuations in a dataset which are not predictive, and a {\em noise} which captures the intrinsic unpredictability even in the case of a perfect model. The decomposition was originally formulated just for the case where the generalization error is quantified in terms of the mean-squared error, although extensions to the 0-1 loss and $\ell_1$ (absolute value) loss also exist \citep{kohavi1996bias, james1997generalizations, domingos2000unified}.

For applications in density estimation and generalized linear models, where the prediction error generally takes the form of a log likelihood, it is necessary to go beyond the squared error. If the noise model takes the form of an exponential family distribution, then the log likelihood can be written in the form of a {\em Bregman divergence} \citep{bregman1967relaxation} (Lemma~\ref{lemma:exp-fam}). This includes the cross-entropy loss frequently used in classification, self-supervised learning and large language modeling. A bias-variance decomposition for exponential family distributions has previously been derived \citep{hansen2000general}. The connection between that result and the main result here was discussed at length in \cite{heskes2025bias}, including a generalization of the result here to g-Bregman divergences. Essentially the same result as the one here has also appeared in the literature on proper scoring rules \citep{buja2005loss}. As the original version of this note on the author's website \citep{pfau2013generalized} has become widely cited in its own right, it is hoped that the citations provided here can put the result in its proper context.

\section{Proof}
\begin{definition}[Bregman Divergence]
Let $F : \mathcal{S} \to \mathbb{R}$ be a strictly convex differentiable function, then the \underline{Bregman Divergence} derived from $F$ is a function $D_F : \mathcal{S} \times \mathcal{S} \to \mathbb{R}_+$ defined as
\[ \DF{x}{y} \triangleq F(x) - F(y) - \inner{\nabla F(y)}{x - y}. \]
\end{definition}

\begin{lemma}[Minimum Expected Bregman Divergence]
Let $F : \mathcal{S} \to \mathbb{R}$ be a strictly convex differentiable function, and $X$ be a random variable on $\mathcal{S}$. Then $x^* = \arg \min_z \E{\DF{z}{X}} \Leftrightarrow \nabla F(x^*) = \E{\nabla F(X)}$ and $\E{X} = \arg \min_z \E{\DF{X}{z}}$.
\end{lemma}

\begin{proof}
A necessary condition for $x^*$ to minimize the expected divergence is that its gradient should be zero. The gradient of the expected Bregman divergence when the expectation is taken over the second argument is given by
\begin{align*}
\nabla_z \E{\DF{z}{X}} &= \nabla_z \E{F(z) - F(X) - \inner{\nabla F(X)}{z - X}} \\
&= \nabla F(z) - \nabla_z \inner{\E{\nabla F(X)}}{z} \\
&= \nabla F(z) - \E{\nabla F(X)} = 0 \\
&\Rightarrow \nabla F(z) = \E{\nabla F(X)}
\end{align*}
by the linearity of expectations and the independence of $z$ from $X$. Since $F$ is convex, if an $x^*$ exists that satisfies this condition then it is unique, and therefore the minimum.

When the expectation is taken over the first argument, the gradient is then
\begin{align*}
\nabla_z \E{\DF{X}{z}} &= \nabla_z \E{F(X) - F(z) - \inner{\nabla F(z)}{X - z}} \\
&= -\nabla F(z) - \nabla \inner{\nabla F(z)}{\E{X}} + \nabla \inner{\nabla F(z)}{z} \\
&= -\nabla F(z) - \nabla^2 F(z) \E{X} + \nabla^2 F(z) z + \nabla F(z) \\
&= -\nabla^2 F(z) \E{X} + \nabla^2 F(z) z = 0 \\
&\rightarrow \nabla^2 F(z) z = \nabla^2 F(z) \E{X} \\
&\rightarrow z = \E{X}
\end{align*}
where the last step follow from the fact that the Hessian of a strictly convex function is positive definite and therefore invertible.
\end{proof}

\begin{theorem}[Decomposition of Expected Bregman Divergence]
Let $F : \mathcal{S} \to \mathbb{R}$ be a strictly convex differentiable function, and $X$ be a random variable on $\mathcal{S}$. Then for any point $s \in \mathcal{S}$, the expected Bregman divergences have the following exact decomposition:
\begin{align*}
\E{\DF{s}{X}} &= \DF{s}{x^*} + \E{\DF{x^*}{X}}, \quad \text{where } x^* = \arg \min_z \E{\DF{z}{X}} \\
\E{\DF{X}{s}} &= \DF{x^*}{s} + \E{\DF{X}{x^*}}, \quad \text{where } x^* = \arg \min_z \E{\DF{X}{z}} = \E{X}.
\end{align*}
\label{thm:bregman}
\end{theorem}

\begin{proof}
\begin{align*}
\DF{s}{x^*} + \E{\DF{x^*}{X}} &= F(s) - F(x^*) - \inner{\nabla F(x^*)}{s - x^*} \\
&\quad + \E{F(x^*) - F(X) - \inner{\nabla F(X)}{x^* - X}} \\
&= F(s) - \inner{\E{\nabla F(X)}}{s - x^*} \\
&\quad + \E{-F(X) - \inner{\nabla F(X)}{x^* - X}} \\
&= \E{F(s) - F(X) - \inner{\nabla F(X)}{s - x^* + x^* - X}} \\
&= \E{F(s) - F(X) - \inner{\nabla F(X)}{s - X}} \\
&= \E{\DF{s}{X}}
\end{align*}

\begin{align*}
\DF{\E{X}}{s} + \E{\DF{X}{\E{X}}} &= F(\E{X}) - F(s) - \inner{\nabla F(s)}{\E{X} - s} \\
&\quad + \E{F(X) - F(\E{X}) - \inner{\nabla F(\E{X})}{X - \E{X}}} \\
&= -F(s) - \inner{\nabla F(s)}{\E{X} - s} \\
&\quad + \E{F(X)} - \inner{\nabla F(\E{X})}{\E{X} - \E{X}} \\
&= \E{F(X) - F(s) - \inner{\nabla F(s)}{X - s}} \\
&= \E{\DF{X}{s}}
\end{align*}
\end{proof}

Suppose we wish to predict some random variable $Y \in \mathcal{S}$ that is dependent on another variable $X \in \mathcal{R}$. We are given a training set $D = \{\{x_1, y_1\}, \dots, \{x_n, y_n\}\}$ of input/output pairs sampled iid from the joint distribution of $X$ and $Y$, and have an algorithm that learns a deterministic prediction function from the data $f_D : \mathcal{R} \to \mathcal{S}$. If the loss function for evaluating the quality of prediction is the Bregman divergence derived from $F$, $L(y, f_D(x)) = \DF{y}{f_D(x)}$ then the expected loss can be decomposed exactly.

\begin{theorem}[Generalized Bias-Variance Decomposition]
Let $F : \mathcal{S} \to \mathbb{R}$ be a strictly convex differentiable function, $f_D : \mathcal{R} \to \mathcal{S}$ be the prediction function trained on data $D = \{\{x_1, y_1\}, \dots, \{x_n, y_n\}\}$, and $Y$ be the random variable we are trying to predict from $X$. Then the expected Bregman divergence of the data obeys a generalized bias-variance decomposition:
\begin{align*}    
\Ex{D,Y}{\DF{Y}{f_D(X)}} = &\Ex{Y}{\DF{Y}{f^*(X)}} \tag{Noise}\\
&+ \DF{f^*(X)}{\bar{f}(X)} \tag{Bias} \\
&+ \Ex{D}{\DF{\bar{f}(X)}{f_D(X)}} \tag{Variance}
\end{align*}
where $f^*(X) = \arg \min_z \Ex{Y}{\DF{Y}{z}} = \Ex{Y}{Y}$, $\bar{f}(X) = \arg \min_z \Ex{D}{\DF{z}{f_D(X)}}$, and all expectations are implicitly conditioned on $X$.
\label{thm:bvd}
\end{theorem}

\begin{proof}
The proof is a straightforward consequence of Theorem~\ref{thm:bregman}.
\begin{align*}
\Ex{D,Y}{\DF{Y}{f_D(X)}} &= \Ex{D}{\Ex{Y}{\DF{Y}{f_D(X)}}|D]} \\
&= \Ex{D}{\Ex{Y}{\DF{Y}{f^*(X)}}|D] + \DF{f^*(X)}{f_D(X)}} \\
&= \Ex{Y}{\DF{Y}{f^*(X)}} + \Ex{D}{\DF{f^*(X)}{f_D(X)}} \\
&= \Ex{Y}{\DF{Y}{f^*(X)}} + \DF{f^*(X)}{\bar{f}(X)} + \Ex{D}{\DF{\bar{f}(X)}{f_D(X)}}
\end{align*}
\end{proof}

\subsubsection*{Acknowledgements}
The author would like to thank Frank Nielsen for helpful pointers to the prior literature and Tom Heskes for feedback.

\bibliographystyle{abbrvnat}
\bibliography{main}

\appendix 

\section{Exponential Families}
\label{sec:equivalence}

\begin{definition}[Exponential Family]
A set of probability distributions $p(x;\eta)$ with $x\in \mathcal{R}$ and natural parameters $\eta \in \mathcal{S}$ form an \underline{exponential family} if the distributions can be written in the form
\begin{equation*}
    p(x; \eta) = h(x) \mathrm{exp}\left(\langle \eta, T(x) \rangle - A(\eta) \right)
\end{equation*}
where $T: \mathcal{R}\rightarrow \mathcal{S}^*$ is a sufficient statistic, $\mathcal{S}^*$ is the dual space to $\mathcal{S}$, $h: \mathcal{R}\rightarrow \mathbb{R}$ is a base measure and $A: \mathcal{S}\rightarrow \mathbb{R}$ is the log partition function or free energy.

The free energy can also be written as
\begin{equation*}
    A(\eta) = \mathrm{log}\left(\int_{\mathcal{R}} dx\, h(x)\mathrm{exp}(\langle \eta, T(x) \rangle \right)
\end{equation*}
It is a convex function of $\eta$. Its gradient is the mean of the sufficient statistics:
\begin{equation*}
    \nabla A(\eta) = \mathbb{E}_\eta\left[T(x)\right] \triangleq \mu
\end{equation*}
while its convex conjugate $A^*: \mathcal{S}^*\rightarrow \mathbb{R}$ is the negative entropy:
\begin{equation*}
    A^*(\mu) \triangleq \sup_{\eta} \langle \mu, \eta \rangle - A(\eta) = \mathbb{E}_\eta \left[\mathrm{log} p(x;\eta)\right]
\end{equation*}
\end{definition}
These are discussed at greater length in \cite{wainwright2008graphical}.

\begin{lemma}[Exponential Family as a Bregman Divergence]
    The log likelihood of an exponential family distribution can be expressed as a Bregman divergence:
    \begin{equation*}
        \mathrm{log}p(x;\eta) = -D_{A^*}[T(x)||\mu] + A^*(T(x))
    \end{equation*}
    \label{lemma:exp-fam}
\end{lemma}
\begin{proof}
    Expanding out the Bregman divergence on the right hand side of the above equation gives
    \begin{align*}
    -D_{A^*}[T(x)||\mu] &= A^*(\mu) - A^*(T(x)) + \langle \nabla A^*(\mu), T(x) - \mu \rangle \\
    &= A^*(\mu) - A^*(T(x)) + \langle \eta, T(x) - \mu \rangle \\
    &= A^*(\mu) - A^*(T(x)) + \langle \eta, T(x) \rangle - \langle \eta, \mathbb{E}_\eta\left[T(x)\right] \rangle \\
    &= A^*(\mu) - A^*(T(x)) + \langle \eta, T(x) \rangle - \mathbb{E}_\eta\left[\langle \eta, T(x) \rangle\right] \\
    &= A^*(\mu) - A^*(T(x)) + \langle \eta, T(x) \rangle - \mathbb{E}_\eta\left[\mathrm{log}p(x;\eta) + A(\eta) - \mathrm{log}h(x)\right] \\
    &= A^*(\mu) - A^*(T(x)) + \langle \eta, T(x) \rangle - A^*(\mu) - A(\eta) + \mathrm{log}h(x) \\
    &= - A^*(T(x)) + \langle \eta, T(x) \rangle - A(\eta) + \mathrm{log}h(x) \\
    \end{align*}
Adding $A^*(T(x))$ to both sides of the equation gives the log likelihood on the right.
\end{proof}

\end{document}